\documentclass[12pt]{article}
\usepackage{arxiv}

\usepackage{amsmath,mathtools}
\usepackage[utf8]{inputenc}
\usepackage{xspace}
\usepackage{url}
\usepackage[numbers]{natbib}
\usepackage{dsfont}
\usepackage{enumerate}

\usepackage{caption}
\usepackage{subcaption}

\usepackage{balance}

\usepackage{amsmath,amssymb,amsthm}
\usepackage{tikz}
\usetikzlibrary{shapes,arrows}
\usetikzlibrary{decorations}
\usetikzlibrary{plotmarks}
\usetikzlibrary{calc}
\usetikzlibrary{mindmap}
\usetikzlibrary{shadows}
\usetikzlibrary{backgrounds}
\usetikzlibrary{shapes}
\usetikzlibrary{shapes.symbols}

\usepackage{pgfplots}
\usepackage{pgfplotstable}
\usepgfplotslibrary{statistics}

\newtheorem{theorem}{Theorem}%[section]
\newtheorem{corollary}{Corollary}%[section]
\newtheorem{lemma}{Lemma}%[section]

\newcommand{\expect}[1]{\mathord{\mathrm{E}}\mathord{\left(#1\right)}}

\newcommand{\filtzero}{\mathcal{F}_0}
\usepackage[algo2e,ruled,vlined]{algorithm2e}
\SetAlgoSkip{}

\allowdisplaybreaks[3]

\newtheorem{openproblem}{Open Problem}

\newcommand*{\om}{\textsc{OneMax}\xspace}
\newcommand*{\OneMax}{\om}
\newcommand*{\onemax}{\om}

\newcommand*{\jump}{\textsc{Jump}\xspace}

\newcommand{\cga}{cGA\xspace}
\newcommand{\cGA}{\cga}

\newcommand*{\Var}{\mathrm{Var}}

\newcommand{\ie}{i.\,e.\xspace}

\newcommand{\eg}{e.\,g.\xspace}

\newcommand{\card}[1]{\left|#1\right|}

\newcommand{\N}{\mathds{N}}
\newcommand{\R}{\mathds{R}}
\newcommand{\filt}{\mathcal{F}}
\newcommand{\filtt}{\mathcal{F}_t}

\renewcommand{\epsilon}{\varepsilon}

\newcommand{\prob}[1]{\mathord{\ensuremath{\mathrm{Pr}}}\mathord{\left(#1\right)}}         % Pr with argument.

\newcommand{\cliff}{\textsc{Cliff}\xspace}
\newcommand{\ones}[1]{\lvert #1\rvert_1}

\newcommand{\wellbehaved}{\mathcal{K}}

\newenvironment{proofof}[1]{\begin{proof}[Proof of~#1]}{\end{proof}}

\newcommand{\vl}{v_{\mathrm{\ell}}}
\newcommand{\vu}{v_{\mathrm{u}}}

\newcommand{\toggleplot}[1]{{\textcolor{red}{Plots removed to increase compilation speed. Use command $\backslash$toggleplot in preamble to reinsert them.}}}

\begin{document}

\author{Frank Neumann
\\Optimisation and Logistics\\
School of Computer Science
\\The University of Adelaide\\Adelaide, Australia
\And
Dirk Sudholt\\
Faculty of Computer Science and Mathematics\\University of Passau\\
Passau, Germany
\And
Carsten Witt\\
DTU Compute\\Technical University of Denmark\\
Kongens Lyngby, Denmark
}

 \
% \email{}
\title{The Compact Genetic Algorithm Struggles on Cliff Functions}

\maketitle

\begin{abstract}
  The compact genetic algorithm (cGA) is a non-elitist estimation of distribution algorithm which has shown to be able to deal with difficult multimodal fitness landscapes that are hard to solve by elitist algorithms. In this paper, we investigate the cGA on the \cliff function for which it has been shown recently that non-elitist evolutionary algorithms and artificial immune systems optimize it in expected polynomial time. We point out that the cGA faces major difficulties when solving the \cliff function and investigate its dynamics both experimentally and theoretically around the \cliff.
  Our experimental results indicate that the cGA requires exponential time for all values of the update strength~$K$. We show theoretically that, under sensible assumptions, there is a negative drift when sampling around the location of the cliff. Experiments further suggest that there is a phase transition for~$K$ where the expected optimization time drops from $n^{\Theta(n)}$ to $2^{\Theta(n)}$.
\end{abstract}


\section{Introduction}
Runtime analysis of evolutionary algorithms and other randomized search heuristics has provided a deep understanding of many working principles of these algorithms~\cite{Jansen13,DoerrN20}.
The goal of these studies is to provide rigorous results of randomized search heuristics by analyzing them as a special class of randomized algorithms. This allows to use a wide range of tools such as concentration bounds and random walk arguments. A wide range of new methods for analyzing randomized search heuristics have been developed over the last 20 years. Starting with methods such as fitness based partitions for simple problems and elitist algorithms, more complex combinatorial optimization problem~\cite{NeumannW10} (including NP-hard ones) and non elitist algorithms have been investigated.

Estimation of distribution algorithms (EDAs)~\cite{PelikanHandbook15} are a special class of randomized search heuristics that work with a probability distribution at each stage of the algorithm (instead of a set of solutions). This probability distribution is updated by reinforcing components that have shown to lead to solutions of good quality. EDAs have found a wide range of applications to problems such as military antenna design, multiobjective knapsack, and  quadratic assignment (see \cite{HauschildPelikan11} for an introduction and overview).

The theoretical runtime analysis concentrates on simple EDAs that capture their basic algorithmic properties~\cite{krejca2019theory}. The compact genetic algorithm (cGA) is such a simple EDA which has been studied in different runtime analyses.
Following the seminal work by Droste~\cite{Droste2006a} for the cGA in the mid 2000s, there has been a growing interest in studying the cGA and other EDAs  over the last 8 years~\cite{DangLehreGECCO15, WittGECCO17,LehreNguyenGECCO17}. We refer the reader to \citet{krejca2019theory} for a recent survey.
These theoretical results focus on the working principles of the considered EDAs and especially discuss their difference to simple evolutionary algorithms such as the (1+1)~EA.
Several studies have shown that the update strength~$K$, which determines the magnitude of changes to the probabilistic model, has a crucial impact on performance~\cite{SudholtWittAlgo19,DoerrZhengIEEETEC20,LenglerSudholtWittAlgo21}. In~\cite{LenglerSudholtWittAlgo21} it was shown that the \cga optimizes \onemax efficiently, in expected time $O(\sqrt{n}K)$, if the update strength is sufficiently large, i.\,e.\ $K = \Omega(\sqrt{n} \log n)$. For $K = \Theta(\sqrt{n}\log n)$ this yields an upper bound of $O(n \log n)$ function evaluations.
%In this regime all frequencies are unlikely to reach their lower border and slowly move up to
In \cite{Lengler2018,LenglerSudholtWittAlgo21}, the authors showed that for smaller values of~$K$ in $\Omega(\log^3 n)$ and $O(\sqrt{n}/(\log(n) \log \log n))$, the expected optimization time on \onemax is $\Omega(K^{1/3} n)$ in expectation and with high probability. Thus, in this so-called medium parameter regime the expected optimization time increases with $K$ before dropping down to $O(n \log n)$ for $K \ge \Omega(\sqrt{n}\log n)$.

Other studies have unveiled remarkable advantages of EDAs. Their ability to learn good solution components, coupled with a slow adaptation of the probabilistic model, makes EDAs higly robust with respect to noisy fitness evaluations~\cite{Friedrich2017}.
%and it has been shown that the cGA is more efficient in such settings than RLS and the (1+1)~EA.
Furthermore, their ability to sample with a large sampling variance implies that they are good at exploring the search space.
This has been shown rigorously for the $\jump$ function, a multimodal function of unitation (i.\,e.\ the fitness only depends on the number of ones) where evolutionary algorithms typically need to make a large jump. With the right choice of the update strength, the \cga is able to optimize \jump efficiently, if the size of the jump is not too large~\cite{DoerrJumpAlgo21, WittFOGA21, HasenohrlS18}.

In this work we consider the runtime of the \cga on a multimodal function.
%We provide a runtime analysis and experimental investigations for the cGA on the pseudo-Boolean function \cliff.
\cliff is a function of unitation with the difficulty that inferior solutions need to be accepted in order to advance towards the global optimum (unless the algorithm jumps to the optimum directly). In other words, algorithms need to be able to ``jump down'' a cliff in the fitness landscape (see Section~\ref{sec:preliminaries} for a definition and Figure~\ref{fig:cliff} for an illustration of \cliff).
It was originally proposed by~\citet{JaegerskuepperFOCI07} to show the advantages of non-elitism in evolutionary algorithms. They showed that a simple (1,$\lambda$)~EA that generates $\lambda$ offspring independently and picks the best offspring to replace the parent optimizes \cliff in expected $O(n^{25})$ evaluations. \Citet{HeviaFajardoSudholtFOGA21} showed that this time is in fact in $O(n^\eta \log^2 n)$ and $\omega(n^{\eta-\varepsilon})$ for every constant $\varepsilon > 0$, where $\eta \approx 3.976770136$.

The same paper~\cite{HeviaFajardoSudholtFOGA21} also showed that a (1,$\lambda$)~EA with a self-adjusting offspring population size~$\lambda$ can optimize \cliff in expected $O(n)$ generations and $O(n \log n)$ expected function evaluations.
The same time bound $O(n \log n)$ was shown earlier for other non-elitist algorithms: hyperheuristics that have a certain probability of accepting every offspring~\cite{LissovoiMultimodal2019}
and for
evolutionary algorithms using ageing~\cite{Corus2020}.
%, for a slightly different location of the cliff.

%With the right selection
% of parameters (\eg, the offspring population size~$\lambda$), these search heuristics can optimize \cliff in
% expected polynomial time.

The \cliff function has a similar structure to \jump with a jump length of $n/3$ and the sets of local optima are identical for both functions. However, when overcoming those local optima, \cliff shows a gradient pointing towards the global optimum whereas \jump has a gradient leading back towards local optima. The gradient structure for \cliff is hence more benign than that for \jump.

Based on the aforementioned positive results for non-elitist evolutionary algorithms on \cliff, and the positive results for the \cga on \jump, one might expect that the non-elitist \cga is also effective on \cliff, if the update strength is chosen just right.

%It has been subject to recent studies of elitist and non-elitist evolutionary algorithms. Based on recent results that show that the cGA is able to deal with JUMP functions where the gap is not too large, one might expect that the cGA can also efficiently optimize CLIFF.

The main contribution of this paper is to show, theoretically and empirically, that this is not the case. In particular, the \cga does not seem to benefit much from the benign gradients past the set of local optima, that is, past the top of the cliff.

By examining the behavior of the cGA when sampling around the cliff, we show in Section~\ref{sec:drift} that, under some conditions, the probabilistic model experiences a negative drift and tends to move away from the optimal distribution. This happens when the \cga tends to sample one offspring at the top of the cliff and one offspring at the bottom of the cliff, and the former offspring is reinforced.
This negative drift prevents the probabilistic model to overcome the region around the cliff, leading to exponential times.
%(see Section~\ref{sec:semirigorousexpoentialbound}). 

Our negative drift bound uses novel arguments for the analysis of the \cga by approximating conditional sampling distributions in the \cga, conditional on 
whether the offspring lie on the same side or on
different sides of the cliff, 
by truncated normal distributions. 
However, this novel approach is not fully rigorous as it is based on the assumption that the sampling variance is always super-constant.

% it has to deal (under certain assumptions on the sampling 
% variance) with a negative drift that leads to an exponential optimization time.  To 
% prove this negative drift and to apply 
% a negative drift theorem, we introduce new arguments for the analysis of the \cga by approximating   
% drifts in conditional spaces
% by truncated normal distributions. 

We conjecture that the variance typically stabilizes to super-constant values and continue to prove exponential lower bounds on the expected optimization time in Section~\ref{sec:semirigorousexpoentialbound} under conjectured lower and upper bounds on the variance. We justify our conjecture in Section~\ref{sec:superconstantvariance} by reviewing related work by~\citet{LenglerSudholtWittAlgo21} on \onemax, where such variance bounds were proven rigorously, and explain which parts of their analysis can be translated to \cliff and where this approach breaks down. We instead present empirical data on the sampling variance to support our conjecture. 

In Section~\ref{sec:experiments} we provide experiments on the runtime of the \cga on \cliff. The parameter landscape for the update strength~$K$ shows a highly complex behavior. Our data suggests that the expected optimization time slowly increases from $2^{\Theta(n)}$ to $n^{\Theta(n)}$ as $K$ grows, before dropping sharply to $2^{\Theta(n)}$ again. We give possible explanations for these effects and finish with a list of open problems.

\section{Preliminaries}
\label{sec:preliminaries}

The \cga is defined in Algorithm~\ref{alg:cga}. It uses a univariate probabilistic model of \emph{frequencies} $p_{t, 1}, p_{t, 2}, \dots, p_{t, n} \in [0, 1]$, which is used to sample new search points. The $i$-th frequency $p_{t, i}$ represents the probability of setting the $i$-th bit to~1 in iteration~$t$. In every iteration, the \cga samples two search points $x$ and $y$ in this way. We shall refer to these as \emph{offspring}, using the language of evolutionary computation. It then sorts $x$ and $y$ such that $f(x) \ge f(y)$ and reinforces $x$ in the probabilistic model. This is done by inspecting the bits at position~$i$ and increasing $p_{t, i}$ if $x_i=1$ and $y_i=0$ and decreasing $p_{t, i}$ if $x_i=0$ and $y_i=1$. The aim is to increase the likelihood of sampling the bit value of the better offspring in the future. If both offspring have the same bit value, the frequency $p_{t, i}$ is unchanged. Frequencies are changed by $\pm 1/K$ and $K$ is called the \emph{update strength} of the \cga. Small values of $K$ imply large values of $1/K$ and hence large changes. This means that novel information has a large impact on the probabilistic model. Large values of~$K$ imply small changes to the probabilistic model, such that the probabilistic model is adapted gradually, and information from many past samples is stored in the frequencies.

Frequencies are always capped to the interval $[1/n, 1-1/n]$ such that the probability of sampling any particular search point is always at least $(1/n)^n > 0$.
%and the \cga is guaranteed to optimize any fitness function in expected time at most $n^n$. 
We refer to $1/n$ as the \emph{lower border} and to $1-1/n$ as the \emph{upper border}. Throughout the paper we tacitly assume that $K$ is in the set $\wellbehaved\coloneqq \{ i(1/2-1/n) \mid i\in \N\}$ so that
 the state space of frequencies is restricted to $p_{i, t} \in \{1/n, 1/n + 1/K, \dots, 1/2, \dots, 1-1/n-1/K, 1-1/n\}$.

As common in theoretical runtime analysis, we define the \emph{optimization time} as the number of function evaluations required to sample a global optimum for the first time. Since the \cga makes two evaluations in every iteration, the optimization time is twice the number of iterations needed to sample a global optimum.

% \cite{BBDExtendedJumpGECCO21}

\begin{algorithm2e}
  $t \gets 0$;\\
  $p_{t,1} \gets p_{t,2} \gets \cdots \gets p_{t,n} \gets 1/2$;\\
  \While{termination criterion not met}{
    \For{$i \in \{1,\dots,n\}$\label{li:x}}{
      $x_i \gets 1$ with prob.\ $p_{t,i}$, $x_i \gets 0$ with prob.\ $1 - p_{t,i}$;
    }
  \For{$i \in \{1,\dots,n\}$\label{li:y}}{
    $y_i \gets 1$ with prob.\ $p_{t,i}$, $y_i \gets 0$ with prob.\ $1 - p_{t,i}$;
  }
  \lIf{$f(x) < f(y)$\label{li:eval}}{swap $x$ and $y$}
  \For{$i \in \{1,\ldots,n\}$}{
    \lIf{$x_i > y_i$}{$p_{t+1,i} \gets p_{t,i} + 1/K$}
    \lIf{$x_i < y_i$}{$p_{t+1,i} \gets p_{t,i} - 1/K$}
    \lIf{$x_i = y_i$}{$p_{t+1,i} \gets p_{t,i}$}
    $p_{t+1,i} \gets \max\{\min\{p_{t+1,i},1-1/n\}, 1/n\}$;
	}
	$t \gets t+1$;
}
\caption{Compact Genetic Algorithm (\cga)}
\label{alg:cga}
\end{algorithm2e}

The function \cliff is a function of unitation, that is, it only depends on the number of ones in a bit string~$x$, denoted as $\ones{x}$. Then \cliff is defined as:
\[
\cliff(x) \coloneqq
\begin{cases}
\ones{x} & \text{if $\ones{x}\le 2n/3$}\\
\ones{x}-n/3+1/2 & \text{otherwise}
\end{cases}
\]
See Figure~\ref{fig:cliff} for an illustration.
We refer to the region of search points with at most $2n/3$ ones as the \emph{first slope}, and all remaining search points as the \emph{second slope}. The only global optimum is the all-ones string $1^n$ with a fitness of $2n/3 + 1/2$. All search points with $2n/3$ ones are local optima at the top of the cliff. Note that all search points on the second slope are strictly worse than all search points at the top of the cliff, except for the global optimum.

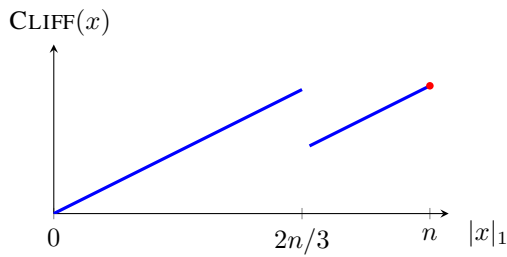
\begin{figure}[ht]
\centering\begin{tikzpicture}[xscale=1,yscale=1.0]
\begin{axis}[
    axis lines=left, xlabel={$\ones{x}$},
   ytick=\empty,
  	ymin = 0, ymax= 9, samples=100,
	xmin=0,
  xmax=10.5,
  x=0.5cm,
  y=0.5cm/2,
  grid = none,
  ylabel=$\cliff(x)$, ylabel style={rotate=-90,at={(0.25,1)}, anchor=south},
  xlabel style={at={(1.1,0.25)},anchor=north},
  xtick = {0,6.6,10},
      xticklabels = {$0$,
        $2n/3$,$n$
      }
	  %xtick style={draw=none},
  ]
	\addplot[blue, very thick, domain=0:6.6] {x};
	\addplot[blue, very thick, domain=6.8:10] {x-3.2};
	\fill[red] (axis cs: 10.0,6.8) circle(0.5mm);
	\end{axis}
\end{tikzpicture}
\caption{Illustration of \cliff}
\label{fig:cliff}
\end{figure}

% \dirk{Some text on related work on \cliff:}
% The function \cliff was originally proposed by~\citet{JaegerskuepperFOCI07} to show the advantages of non-elitism in evolutionary algorithms. They showed that a simple (1,$\lambda$)~EA that generates $\lambda$ offspring independently and picks the best offspring to replace the parent optimizes \cliff in expected $O(n^{25})$ evaluations. \Citet{HeviaFajardoSudholtFOGA21} gave a more precise analysis and showed that this time is in fact in $O(n^\eta \log n)$ and $\omega(n^{\eta-\varepsilon})$ for every constant $\varepsilon > 0$, where $\eta \approx 3.976770136$.

% The same paper~\cite{HeviaFajardoSudholtFOGA21} also showed that a (1,$\lambda$)~EA with a self-adjusting offspring population size~$\lambda$ can optimize \cliff in expected $O(n)$ generations and $O(n \log n)$ expected function evaluations.
% The same time bound $O(n \log n)$ was shown earlier for hyperheuristics that have a certain probability of accepting every offspring~\cite{LissovoiMultimodal2019}
% and for
% evolutionary algorithms using ageing~\cite{Corus2020}, for a slightly different location of the cliff.

% With the right selection
% of parameters (\eg, the offspring population size~$\lambda$), these search heuristics can optimize \cliff in
% expected polynomial time.

%\subsection{Tools for the Analysis}
When analyzing the \cGA on functions of unitation (\eg, \OneMax
as analyzed in \cite{SudholtWittAlgo19,LenglerSudholtWittAlgo21} and \jump as
analyzed in \cite{DoerrJumpAlgo21,WittFOGA21}), one
is interested in the number of one-bits sampled in an
offspring. This random value follows a Poisson-binomial
distribution with the frequencies $(p_{t,1},\dots,p_{t,n})$
as underlying success probabilities. In particular,
the following two quantities play a key role in bounding
the progress of the \cGA  towards the optimum:

\begin{enumerate}
    \item the \emph{potential} $P_t\coloneqq \sum_{i=1}^n p_{t,i}$  equals the expected value of the Poisson-binomial distribution, \ie, the expected number
    of one-bits
    sampled in an offspring,
    \item the \emph{sampling variance} $V_t \coloneqq
    \sum_{i=1}^n p_{t,i}(1-p_{t,i})$ is the variance in the number
    of one-bits.
\end{enumerate}

The following negative drift theorem
will be used in Section~\ref{sec:semirigorousexpoentialbound} to analyze the one-step change of potential
$\Delta_t\coloneqq P_{t+1}-P_t$.

\begin{theorem}[Negative Drift with Scaling, cf.~\cite{OlivetoWittTCS15}]\label{theo:simplified-drift-scaling}
  Let $(X_t)_{t\ge 0}$
  be a stochastic process, adapted to a filtration
$\filtt$, over some state space $S\subseteq \R$.
	 Suppose
  there exist an interval $[a,b]\subseteq \R$
   and, possibly depending on
  $\ell:=b-a$, a drift bound $\epsilon:=\epsilon(\ell)>0$
	as well as a scaling factor $r:=r(\ell)>0$
   such that
  for all $t\ge 0$ the following three conditions hold:
  \begin{enumerate}
  \item $\expect{X_{t+1}-X_{t}\mid \filtt \,;\, a< X_t <b} \ge \epsilon$,
  \item $\prob{\lvert X_{t+1}-X_t\rvert\ge jr \mid \filtt  \,;\, a< X_t}  \le  e^{-j}$ for $j\in \N_0$,
  \item
  $1\le r^2 \le \epsilon\ell/(132\log (r/\epsilon))$.
  \end{enumerate}
  Then for the first hitting time  $T^*:=\min\{t\ge
  0\colon X_t\le a \mid X_0\ge b\}$ it holds that $\prob{T^*\le
  e^{\epsilon\ell/(132r^2)}\mid \filtzero}= O(e^{-\epsilon\ell/(132r^2)})$.
\end{theorem}

To verify the second condition of the negative drift theorem in
our concrete analysis, we will use the following
lemma dealing with Chernoff-type bounds depending on the
variance. The lemma goes
back to \cite{Hoeffding1956}. We present a version
given in \cite[Theorem 1.10.14]{DoerrProbabilisticTools}.
\begin{lemma}
\label{lem:hoeffding-variance}
Let $X_1,\dots,X_n$ be independent random variables. Let~$b$ be such that $X_i\le \expect{X_i}+b$
for all $i=1,\dots,n$. Let $X=\sum_{i=1}^n X_i$. Let $\sigma^2=\sum_{i=1}^n \Var(X_i) = \Var(X)$.
Then, for all $\lambda\ge 0$,
\[
\prob{X\ge \expect{X} + \lambda} \le e^{-(1/3)\min\{\lambda^2/\sigma^2, \lambda/b\}}.
\]
\end{lemma}

As a simple consequence, we obtain the following corollary:
\begin{corollary}
Consider the \cGA on an arbitrary fitness function 
Then for all $t\ge 0$ and
$\lambda>0$ it holds that
\label{cor:jump-pt-onestep}
\begin{equation*}
\prob{\card{P_{t+1}-P_t}\ge \lambda/K} \le 2e^{-(1/3)\min\{\lambda^2/V_t, \lambda\}}.
%\label{eq:jump-pt-onestep}
\end{equation*}
\end{corollary}

To see that the corollary follows, we argue in the same
way as in \cite{WittFOGA21}, where jump functions were considered:  the absolute value of the one-step change in potential
is no larger than the absolute difference in the number
of one-bits
 of the two individuals sampled,
 scaled down by $1/K$. This holds since each bit
 sampled~$1$ in the fitter offspring and~$0$ in
 the other offspring contributes
 a $+1/K$ to the change of potential (or nothing, in case the
 frequency is capped at the upper border) and no less than
 $-1/K$ in
 the opposite case. The factor~$2$ accounts for the two
 possible orderings of offspring. 
%  We remark that
% the corollary actually applies
%  to all fitness functions.

%
%(since $\card{P_{t+1}-P_t} \le (1/K) \card{X_{t+1}-P_t}$, where $X_{t+1}$ is the number of one-bits of the selected offspring):

\section{Negative Drift Around the Cliff}
\label{sec:drift}

We will under certain assumptions prove
that the potential of the
\cga cannot overcome the cliff region efficiently since
there is a negative drift in the potential. The intuition is as follows:

The initial potential is $n/2$ and,
as long as the potential is significantly less than
$2n/3$, the \cga is very unlikely to sample search
points on the second slope of \cliff. If that
does not happen, the
fitness landscape is the same as on \onemax. Hence,
using the results from \cite{SudholtWittAlgo19},
the potential
$P_t$ will steadily increase towards the location of the cliff, i.\,e.\ $2n/3$.

However, when the potential~$P_t$ has increased to roughly~$2n/3$, \ie,
the expected number of ones sampled is close
to the cliff, it is
relatively likely that the \cga samples search points
on both the first slope and the second slope of
cliff. In particular, if the sampling variance~$V_t$
is large and~$P_t=2n/3$, the sampling distribution
is similar to a normal distribution with mean~$2n/3$
and the given variance. Hence, we are
confronted with an approximately symmetric distribution. Then
the probability of sampling the two offspring on both
sides of the cliff becomes roughly $(1/2)\cdot (1/2)+(1/2)\cdot (1/2)=1/2$ by counting the two opposite
events of sampling the first offspring on the first
slope and the second one on the second slope and vice
versa. By a similar argumentation, also the probability of sampling both
offspring on the same slope will approach $1/4+1/4=1/2$.

We will analyze the drift, \ie, expected change of potential,  under event $M$ of sampling on two different
slopes and its complement. The key observation is
that under~$M$, the offspring with the smaller number
of one-bits will have roughly $2n/3-\sqrt{V_t}$ one-bits
in expectation and the other offspring will have roughly $2n/3+\sqrt{V_t}$
one-bits in expectation by properties of truncated
normal distributions that arise under~$M$.
Since the offspring on the first slope will be
fitter and reinforced in the frequency update, this corresponds to an expected decrease in potential
of $(2n/3-\sqrt{V_t} - (2n/3+\sqrt{V_t}))/K = -2\sqrt{V_t}/K$.

Under $\overline{M}$, both offspring are on the same
slope and their expected difference in one-bits
is no larger than the variance~$\sqrt{V_t}$, again
by simple analyses of truncated normal distributions.
Taking
these two cases of roughly identical probability
together, the total drift becomes $(1/2) (-2\sqrt{V_t} +
\sqrt{V_t})/K = -\sqrt{V_t}/(2K)$.
This argumentation can be made rigorous not only
when $P_t=2n/3$, but for roughly all
$P_t\in [2n/3-\sqrt{V_t},2n/3+\sqrt{V_t}]$, as
the following lemma shows. We will use this result
when applying a negative drift theorem  (Theorem~\ref{theo:simplified-drift-scaling}) in
Section~\ref{sec:semirigorousexpoentialbound}.

\begin{theorem}
\label{the:negative-drift-cliff}
Assume $V_t = \omega(1)$. 
%\ge \alpha(n)$ for some function $\alpha(n)=\omega(1)$.
Let $\epsilon>0$ be an arbitrary constant. Then
conditioned on $P_t\in [2n/3-(\alpha(n))^{1/2-\epsilon},  2n/3]$, it holds that $\expect{\Delta_t\mid P_t} = -\Omega(\sqrt{V_t}/K)$.
\end{theorem}

Before we proceed with the proof, we collect  well-known properties of the expected value 
$\expect{X\mid X\le t}$  of a truncated normal distribution and show that $t-\expect{X\mid X\le t}$, \ie, the 
distance of this expected value from 
the truncation parameter~$t$, increases when the 
truncation condition becomes weaker, \ie, 
when~$t$ grows.

\begin{lemma}
\label{lem:truncated-normal}
Given a normally distributed random variable~$X$ with mean~$\mu$ and variance $\sigma$,
we have for all $t\in \R$ that
\begin{align*}
\expect{X\mid X\le t}  & = \mu - \sigma \frac{\phi((t-\mu)/\sigma)}{\Phi((t-\mu)/\sigma)}\notag\\
\text{\quad and\quad}  \expect{X\mid X\ge t}  & = \mu +  \sigma \frac{\phi((t-\mu)/\sigma)}{\Phi((t-\mu)/\sigma)},
\label{eq:normal-deviation}
\end{align*}
where $\phi$ and $\Phi$ denote the density and cumulative distribution function
of the standard normal distribution, respectively.
Moreover,
the function
$
t-\expect{X\mid X\le t}
$
is monotone increasing in~$t$.
\end{lemma}

\begin{proof}
The first two claims relate to the
expected value of the so-called truncated
normal distribution and are well known
in the literature (\eg, p.~156 in  \cite{JKBContinuousUnivariate}).
%,follows by scaling and
%centering~$X$ to obtain a standard
%normally distributed random variable~$Z$
%and computing its
%conditional distribution.

For the final claim, we consider w.\,l.\,o.\,g.\ a standard 
normally distributed random variable~$Z$ and 
write for arbitrary $x\in\R$ 
\begin{equation}
x-\expect{Z\mid Z\le x} = x + \frac{\phi(x)}{\Phi(x)}
\label{eq:diff-truncated}
\end{equation}
The 
function
$\frac{\phi(x)}{\Phi(x)}$  is known
as the \emph{inverse Mills ratio} in the
literature and known to have a derivative
of at least~$-1$ (see \cite{Sampford1953}, who shows that the derivative of $\frac{\phi(x)}{1-\Phi(x)} = \frac{\phi(-x)}{\Phi(-x)}$ is at most~$1$).
Hence, the derivative of \eqref{eq:diff-truncated}
is at least~$0$ and the
final claim follows.
%the
%Uses elementary representations of %conditional distribution function etc.
%See, \eg,
%\url{https://stats.stackexchange.com/questions/166273/expected-value-of-x-in-a-normal-distribution-given-that-it-is-below-a-certain-v}
\end{proof}

%The following lemma establishes a negative drift in a region of size roughly $\Omega((V_t)^{1/2-\epsilon})$ around the cliff.

\begin{proofof}{Theorem~\ref{the:negative-drift-cliff}}
By assumption, we
have that $V_t=\omega(1)$ for all $t\ge 0$.
Hence, by the generalized central limit theorem
(Chapter~XV.6 in~\cite{Feller2})
the number of one-bits sampled in each offspring, which
follows a Poisson-binomial distribution with
mean~$P_t$ and variance~$V_t$, converges in distribution to a normal distribution with mean $P_t$ and variance $V_t$. More precisely, let  
 $X=\ones{x}$ for an arbitrary offspring sampled with current frequency vector of potential $P_t$ and variance~$V_t$ and 
 let $X'\sim N(P_t,\sqrt{V_t})$. Then 
 for all $t\in \R$, $\prob{X\le t} = (1\pm o(1)) \prob{X'\le t}$.  
 Often, we will pretend that $X\sim N(P_t,\sqrt{V_t})$ and omit $1-o(1)$ factors stemming from the normal approximation.

%\todo{Carsten: explain the error}

We will decompose the drift according to three events for the location of the two offspring of the \cGA:
\begin{enumerate}
    \item[L)] Both offspring have at most $2n/3$ one-bits, \ie, lie both on the first (\textbf{l}eft) slope.
    \item[R)] Both offspring have at least $2n/3+1$ one-bits, \ie, lie both on the second (\textbf{r}ight) slope.
    \item[M)] One offspring has at most $2n/3$ one-bits and one at least $2n/3+1$ one-bits, \ie, there is an offspring on each slope (the \textbf{m}ixed case).
\end{enumerate}

Obviously, by the law of total probability,
\begin{align*}
    \expect{\Delta_t\mid P_t} & = \expect{\Delta_t\mid P_t; L} \prob{L\mid P_t} + \expect{\Delta_t\mid P_t;R} \prob{R\mid P_t} 
     +
    \expect{\Delta_t\mid P_t;M} \prob{M\mid P_t}.
    \label{eq:drift-deltat-pt}
\end{align*}
For readability, we may omit the condition on the random $P_t$ in the following. 
Let $p_R := \prob{X > 2n/3 \mid P_t}$, then $\prob{R} = p_R^2, \prob{L} = (1-p_R)^2$ and $\prob{M} = 2p_R(1-p_R)$. Hence, 
\begin{equation}
    \expect{\Delta_t\mid P_t}  =  \expect{\Delta_t\mid L} (1-p_R)^2 + \expect{\Delta_t\mid R} p_R^2
     + \expect{\Delta_t\mid M} 2p_R(1-p_R).
     \label{eq:drift-deltat-pt-with-pR}
\end{equation}

Let us consider the generation of one offspring more closely, assuming a fixed $P_t$.
A crucial insight, implied by the normal approximation, is that $p_R = \prob{X>2n/3}$ is monotone increasing in $P_t$  (up to multiplicative
errors of $1-o(1)$) and approaches~$1/2$. Even more, already if $P_t=2n/3-(V_t)^{1/2-\epsilon}$ for some constant $\epsilon>0$, the probability $\prob{X>2n/3}$ becomes at least $1/2-o(1)$ 
using the normal approximation. This follows
since 
the density is at most  $\frac{e^{-1/2}}{\sqrt{V_t}\sqrt{2\pi}}= O(1/\sqrt{V_t})$ 
so that $\prob{2n/3 - (V_t)^{1/2-\epsilon} \le X \le \ 2n/3} \le V_t^{1/2-\epsilon} \cdot O(1/\sqrt{V_t}) = o(1)$. 
%if $c$ is chosen small enough but constant.
%This again follows from the properties of the normal distribution.

Let us now fix $c>0$ such that $P_t\in [2n/3-c(V_t)^{1/2-\epsilon},2n/3]$ and $p_R = \prob{X>2n/3}\ge 1/2-1/(V_t)^{1/2-\epsilon} = 1/2-o(1)$. Since $P_t\le 2n/3$, we also have $p_R \le 1/2$ and therefore
\begin{itemize}
    \item $\prob{R} = p_R^2 \le  1/4$
    \item $\prob{M} = 2p_R(1-p_R) \ge 2(1/2-o(1))(1/2)=1/2 - o(1)$
    \item $\prob{L} = (1-p_R)^2 \le (1/2+o(1))^2 = 1/4 + o(1)$.
\end{itemize}

We next estimate the drift under the three events. To this end, we need bounds on the two conditional expectations $\expect{X\mid X\le 2n/3}$ and $\expect{X\mid X\ge 2n/3+1}$ since the conditions
specify that an offspring is on the first and second slope, respectively. Using Lemma~\ref{lem:truncated-normal} with $\mu=P_t$, $\sigma=V_t$ and $t=2n/3$, we have
\begin{align*}
\expect{X\mid X\le 2n/3} & = P_t - \sqrt{V_t} \cdot \frac{\phi((2n/3-P_t)/\sqrt{V_t})}{\Phi((2n/3-P_t)/\sqrt{V_t})}
%\expect{X\mid X\le P_t} + \expect{X\mid P_t\le X\le 2n/3}\prob{P_t\le X\le 2n/3} \\
%& \le  \expect{X\mid X\le P_t} \prob{X\le P_t} + O((V_t)^{1/2-\epsilon}).
\end{align*}
We note that $(2n/3-P_t)/\sqrt{V_t} = O(1/V_t^{\epsilon})=o(1)$ by our choice of $V_t$. Hence,
we have
\begin{equation}
    \label{eq:phi-by-Phi}
\frac{\phi((2n/3-P_t)/\sqrt{V_t})}{\Phi((2n/3-P_t)/\sqrt{V_t})}  =  \frac{\phi(o(1))}{\Phi(o(1))} 
 =  \frac{(1\pm o(1))\phi(0)}{(1\pm o(1))\Phi(0)}
 =  (1\pm o(1)) \sqrt{\frac{2}{\pi}},
\end{equation}
using the continuity of the density and distribution functions in 
the second step and the well-known equality $\frac{\phi(0)}{\Phi(0)}=\sqrt{2/\pi}$ stemming from the
half-normal distribution in the third step. Together,
\[
\expect{X\mid X\le 2n/3} = P_t - (1+o(1)) \sqrt{2/\pi} \sqrt{V_t}.
\]

In the very same way, we derive
\[
\expect{X\mid X > 2n/3} = P_t - (1-o(1)) \sqrt{2/\pi} \sqrt{V_t}.
\]

%The last inequality follows since every point in the interval $[P_t,2n/3]$ has a density of $O(1/\sqrt{V_t})$ and the interval has length $(V_t)^{1/2-\epsilon}$.

Under the event $M$ defined above, we have one offspring with at most $2n/3$ one-bits and another one with strictly more one bits. The update will reinforce the individual on the left slope and change the potential by the difference in the number of one-bits,
divided by~$K$, assuming no frequencies at the border. To correct this for the boundary effects, we apply Lemma~8 in \cite{DoerrJumpAlgo21} and obtain 
an error term of at most $2/K$ in 
the expected change of potential. (Roughly speaking, this accounts for the fact that 
every frequency at the border flips with 
probability at most $2(1/n)(1-1/n)$ and 
that capping reduces its change by at most~$2/K$.) 
Hence, we obtain
\begin{align*}
\expect{\Delta_t\mid M}  & \le - \frac{1}{K}\bigl(\expect{X\mid X > 2n/3} - \expect{X\mid X\le 2n/3}\bigr) + \frac{2}{K}\\
& = - \frac{1}{K}\left(\left(P_t +
(1-o(1)) \sqrt{\frac{2}{\pi}V_t} \right)
%& \qquad
 - \left(P_t- (1+o(1)) \sqrt{\frac{2}{\pi}V_t}\right)\right) + \frac{2}{K} \\
&  
= -(2-o(1))\frac{1}{K}\sqrt{\frac{2}{\pi}V_t},
\end{align*}
where we have used that $V_t=\omega(1)$.

We are left with the drift under~$L$ and~$R$; we only analyze~$L$ since both cases
are analogous. Here we sample two offspring conditional on both having at most $2n/3$ one-bits. 
Let~$X_1$ and $X_2$ denote the random number of one-bits of two offspring and assume w.\,l.\,o.\,g.\ that $X_1\le X_2$. Similarly as for $\expect{\Delta_t\mid M}$,  the potential drift is then % To compute the drift under~$L$, let~$X_1$ and $X_2$ denote the random number of one-bits of two offspring and assume without loss
% of generality that $X_1\le X_2$. We have
% We are left with the drift under~$L$ and~$R$.  %\todo{Can we merge~$L$ and $R$?}
% We analyze only case~$L$ since both cases
% are analogous. Here we sample two offspring $x_1$ and $x_2$ with $X_1$ and $X_2$ one-bits, respectively, conditional on both having at most $2n/3$ one-bits. %
% W.\,l.\,o.\,g.\ $X_1 \le X_2$, then $x_2$ is being reinforced. Let $a$ be the number of positions where $x_1$ has a 1 and $x_2$ has a 0 and $b$ be the number of positions where $x_1$ has a 0 and $x_2$ has a 1. Then $a$ frequencies are decreased by $1/K$, $b$ frequencies are increased by $1/K$ and all other frequencies are unchanged since $x_1$ and $x_2$ agree. Along with $b - a = X_2 - X_1 \ge 0$, the potential is changed by $(X_2-X_1)/K$ and
% Intuitively, compared to the case~$M$ analyzed above, 
% this gives a lower expected difference in the number
% of one-bits  since the two offspring 
% are sampled under the same condition, whereas 
% case~$M$ implies that the offspring are on 
% different slopes and therefore typically apart 
% by some distance.
%
% To compute the drift under~$L$, let~$X_1$ and $X_2$ denote the random number of one-bits of two offspring and assume without loss
% of generality that $X_1\le X_2$. We have
\[
\expect{\Delta_t\mid L} \le \frac{1}{K} \cdot\expect{X_2-X_1\mid X_1\le X_2\le 2n/3}+\frac{2}{K}.
\]
Compared to the case~$M$ analyzed above, the difference $X_2 - X_1$ in the number of one-bits tends to be smaller
since the two offspring 
are sampled on the same slope, whereas 
under~$M$ the offspring are on 
different slopes and $X_2-X_1$ is typically larger.
% Compared to the case~$M$, 
% under $L$ the difference $X_2 - X_1$ is smaller since the two offspring 
% are sampled on the same slope, whereas 
% case~$M$ implies that the offspring are on 
% different slopes and therefore typically apart 
% by some distance.
Using that $X_1$ is normally distributed with variance $V_t$,
Lemma~\ref{lem:truncated-normal} implies
\[
\expect{X_1\mid X_1\le s} = P_t - \sqrt{V_t} \cdot \frac{\phi((s-P_t)/\sqrt{V_t})}{\Phi((s-P_t)/\sqrt{V_t})} ,
\]
where we identify $s=X_2$, assuming $X_2\le 2n/3$. Hence,
\[
\expect{s-X_1\mid s; \,X_1\le s} = s - P_t + \sqrt{V_t} \cdot \frac{\phi((s-P_t)/\sqrt{V_t})}{\Phi((s-P_t)/\sqrt{V_t})}
\]
As shown in Lemma~\ref{lem:truncated-normal},
the right-hand side of the last equation is monotone increasing in~$s$.  Hence,
\begin{align*}
\expect{X_2-X_1\mid X_1\le X_2\le 2n/3} 
 \le\;& \expect{2n/3-X_1 \mid X_1\le 2n/3} \\
=\;& 2n/3 - P_t + \sqrt{V_t} \cdot \frac{\phi((2n/3-P_t)/\sqrt{V_t})}{\Phi((2n/3-P_t)/\sqrt{V_t})}\\
\intertext{(simplifying the fraction using~\eqref{eq:phi-by-Phi})}
% \end{align*}
%
% Recall from~\eqref{eq:phi-by-Phi} that
% $\frac{\phi((2n/3-P_t)/\sqrt{V_t})}{\Phi((2n/3-P_t)/\sqrt{V_t})} = (1+o(1))\sqrt{2/\pi}$,
% implying
% \begin{align*}
%& \expect{X_2-X_1\mid X_1\le X_2\le 2n/3} \\
\le\;& 2n/3 - P_t + (1+o(1)) \sqrt{(2/\pi)V_t} \\
 =\;& (1+o(1)) \sqrt{(2/\pi)V_t} 
\end{align*}
so, since $V_t=\omega(1)$, we have both
\begin{align*}
    \expect{\Delta_t\mid L} & \le  \frac{(1+o(1))}{K} \sqrt{\frac{2}{\pi}V_t}  
    \ \ \text{ and } \ \ \expect{\Delta_t\mid R}  \le \frac{(1+o(1))}{K} \sqrt{\frac{2}{\pi}V_t}.
\end{align*}

Plugging the above bounds in \eqref{eq:drift-deltat-pt-with-pR}, we obtain
\begin{align*}
  \expect{\Delta_t\mid P_t}
   \le\;& \frac{1}{K}\Biggl( (p_R^2 + (1-p_R)^2) \cdot (1+o(1)) \sqrt{(2/\pi)V_t} +
  2p_R(1-p_R) (-(2-o(1)) \sqrt{(2/\pi)V_t} ) \Biggr)\\
  =\;& \frac{1}{K}\left( (1/2 + o(1)) \sqrt{(2/\pi)V_t} -
  (1/2 - o(1)) \cdot 2\sqrt{(2/\pi)V_t} ) \right)\\
  =\;& -(1/2-o(1))\frac{1}{K}\sqrt{(2/\pi)V_t}.
\end{align*}
as claimed.
\end{proofof}

%\dirk{@Carsten, have I split ``your'' part at the right place?} YES

\section{Justifying the Assumption of Super-Constant Sampling Variance}
\label{sec:superconstantvariance}

The approximation with (truncated) normal distributions used in the proof of the drift estimate from Theorem~\ref{the:negative-drift-cliff} hinges on the sampling variance being $\omega(1)$.
We now try to convince the reader why we believe that the sampling variance is $\omega(1)$, for interesting~$K$.
%values of the update strength~$K$.

\subsection{Rigorous Variance Bounds for \onemax}

To this end, we first discuss the variance on \textsc{OneMax}, for which rigorous bounds of $\omega(1)$ have been shown by~\citet{LenglerSudholtWittAlgo21} in a medium parameter regime for~$K$ (as will be defined below).
% The lower bound on the variance is $V_t = \Omega(K^{1/2})$ and it applies with high probability after an initial phase of $O(K^2 \log^2 n)$ iterations. It was further shown that the variance with high probability increases to $V_t = \Omega(K^{2/3})$ after $O(K^2 \log^2(n) \log \log K)$ iterations, however we only discuss the bound of $V_t = \Omega(K^{1/2})$.
The following statement was implicitly shown in~\cite{LenglerSudholtWittAlgo21} and can be deduced from~\cite[Lemma~18]{LenglerSudholtWittAlgo21} and its proof. A discussion will follow.
\begin{theorem}\sloppy
\label{the:lower-bound-on-variance-onemax}
Consider the \cga on \textsc{OneMax} with $K = \Omega(\log^3 n)$ and $K = O(n^{1/2}/(\log(n)\log\log(n)))$.
With probability $1-e^{-\Omega(K^{1/4})}$, there exist times $t_1 = O(K^2 \log^2 n)$, $t_2 = O(K^2 \log^2(n) \log \log K)$ and $t_3 > t_2$ such that the following statements hold.
\begin{enumerate}
\item For all $t \in [t_1, t_2]$, $V_t \ge \Omega(K^{1/2})$.
\item The number of frequencies at the lower border at time $t_2$ is $\Omega(n)$. For all $t_3 > t_2$ such that there are $\Omega(n)$ frequencies at the lower border at all times in $(t_2, t_3)$, with probability $1 - O(t_3/(K^2\log^2 n)) \cdot \exp(-\Omega(K^{1/3}))$,
\[
V_{t} \in \Omega(K^{2/3}) \cap O(K^{4/3}).
\]
\end{enumerate}
\end{theorem}
There is an initial phase of the first $O(K^2 \log^2n) = o(n)$ iterations for which no lower bound on the variance is shown in~\cite{LenglerSudholtWittAlgo21}.
After this phase, we have a lower bound of $\Omega(K^{1/2})$ on the variance that quickly improves to a lower bound of $\Omega(K^{2/3})$. The latter bound applies with good probability as long as there are still $\Omega(n)$ frequencies at the lower border.

We describe the main idea behind the analysis in~\cite{LenglerSudholtWittAlgo21}, and the proof of Theorem~\ref{the:lower-bound-on-variance-onemax}.
First we observe that a frequency at a border contributes only $1/n \cdot (1-1/n)$ to the variance, while frequencies that are off their borders contribute a much larger amount. Hence bounding the variance is achieved by studying the number and position of off-border frequencies.

Lemma~18 in~\cite{LenglerSudholtWittAlgo21} considers the situation after the first $O(K^2)$ iterations, when a linear number of frequencies has reached the lower border, with probability $1-e^{-\Omega(K^{1/2})}$.
Then the authors consider periods of $\Theta(K^2 \log^2 n)$ iterations and show that in a period, frequencies tend to leave their borders to perform a random walk. This random walk ends when a border is reached. (The frequency may then start another random walk during the period.) Frequencies that perform a random walk contribute a term of $p_{i, t}(1-p_{i, t})$ to the sampling variance. %($p_{i, t}$ denotes the $i$-th frequency at time~$t$).
Hence the variance in future iterations can be bounded by analyzing these random walks. The dynamics are intricate since the random walks show a positive drift that depends on the current sampling variance. The drift has a potentially significant impact on the random walks; for instance, it can decide whether a random walk started at the lower border crosses the whole range $[1/n, 1-1/n]$ and ends up at the upper border, or whether it returns to the lower border.
%and the random walks determine the sampling variance in future iterations.
\citet{LenglerSudholtWittAlgo21} argue that the \cga experiences a feedback loop since the current sampling variance influences future sampling variances. This feedback loop has a considerable lag as the effects of a small or large sampling variance are felt at later stages of the frequencies' random walks.

One idea from~\cite{LenglerSudholtWittAlgo21} is to assume that we have lower and upper bounds on the sampling variance during a period as this can then be used to bound the drift for the frequencies' random walks from above and below, and to establish bounds for the sampling variance in the next period. This is formalized in the so-called stabilization lemma, Lemma~7 in~\cite{LenglerSudholtWittAlgo21}, in which lower and upper bounds on the variance in one period are used to show tighter lower and upper bounds in the next period. Part (a) of Lemma~7 in~\cite{LenglerSudholtWittAlgo21} assumes trivial bounds on the sampling variance and applies after the short, initial phase of $O(K^2)$ steps, when $\Theta(n)$ frequencies have reached the lower border.
Part (a) of the stabilization lemma then yields that after a further period of $CK^2 \log^2 n$ iterations, $C > 0$ a sufficiently large constant, the variance is guaranteed to be at least $\Omega(K^{1/2})$ for the next at least period of $CK^2 \log^2 n$ steps, with probability at least $1-e^{-\Omega(K^{1/2})}$.
Lemma~18 in~\cite{LenglerSudholtWittAlgo21} then applies Part (b) of the stabilization lemma iteratively to obtain tighter lower and upper bounds.
More specifically, after $O(\log \log K)$ periods, the variance is guaranteed to be in $\Omega(K^{2/3})$ and $O(n^{4/3})$.
Since the number of frequencies at the lower border only changes very slowly (in comparison to the length of a period), we still have $\Theta(n)$ frequencies at the lower border at this point in time. While this is the case, the stabilization lemma can still be applied to show that these variance bounds are maintained. Each application of the stabilization lemma has a failure probability of $\exp(-\Omega(K^{1/3}))$, thus a union bound over $t_3/(CK^2 \log^2 n)$ applications of the lemma yields the probability bound stated in Theorem~\ref{the:lower-bound-on-variance-onemax}.

\subsection{Trying to Translate Results to Cliff}

We conjecture that Theorem~\ref{the:lower-bound-on-variance-onemax} also holds when replacing \onemax with \cliff. We do not have a proof for this statement, but we will argue why this conjecture seems plausible, and what the challenges are in translating results from~\cite{LenglerSudholtWittAlgo21} on \onemax to \cliff.

Both \cliff and \onemax are functions of unitation, hence on both functions the dynamics can be analyzed by considering individual frequencies. On both functions, frequencies are likely to reach borders within $O(K^2)$ iterations and then frequencies may detach from their borders to perform a random walk. 
%As explained above and in~\cite{LenglerSudholtWittAlgo21}, these random walks essentially determine the sampling variance in the next period. 
Thus, the approach from~\cite{LenglerSudholtWittAlgo21} that leads to the stabilization lemma can also be applied to \cliff.

%To understand the challenges when translating arguments from \onemax to \cliff, we need to inspect the random walks of frequencies that have left their border in more detail.
These random walks are similar for both functions. If we pick an arbitrary but fixed frequency~$i$ then, for both functions, there are steps in which that frequency has no effect on the selection of the fitter offspring and then increasing the frequency has the same probability as decreasing it. These steps were called \emph{random walk steps} in~\cite{SudholtWittAlgo19}. There are other steps, called \emph{biased steps} in~\cite{SudholtWittAlgo19}, in which a frequency can only increase on \onemax. This happens, for instance, when all other bits have the same number of ones in both samples $x$ and $y$ and then the $i$-th frequency determines which search point is reinforced. If exactly one of $x_i$ and $y_i$ is~1, that solution is reinforced and the $i$-th frequency increases.
The probability of a biased step is $\Theta(1/\sqrt{V_t})$ and the expected drift of the $i$-th frequency is $\Theta\left(\frac{p_{i, t}(1-p_{i, t})}{K\sqrt{V_t}}\right)$ if no border is hit.

On \cliff, the situation is similar. If the $i$-th frequency has no impact on selection, a random walk step occurs. Biased steps may occur when all other bits have the same number of ones on both samples or when the number of ones on all other bits is precisely $2n/3$. In the latter case, the $i$-th frequency may decide whether a sample has $2n/3$ ones (i.\,e.\ is on top of the cliff) or $2n/3 + 1$ ones (i.\,e.\ is at the bottom of the second slope). It is not difficult to show that the probability of a biased step is $\Theta(1/\sqrt{V_t})$ as for \onemax.
%We also have that for \onemax and \cliff frequencies are not independent, but dependencies between bits are very mild.

One key difference is that on \cliff the drift can be either positive or negative. It is positive when the \cga focuses its search on one particular slope. However, the drift can be negative when sampling close to the cliff and the two offspring lie on different slopes.
%event $M$ from Section~\ref{sec:drift} occurs. 
%From~\cite{LenglerSudholtWittAlgo21} and Theorem~\ref{the:negative-drift-cliff} it is clear that the absolute value of the drift is bounded by $O\left(\frac{p_{i, t}(1-p_{i, t})}{K\sqrt{V_t}}\right)$ as for \onemax, but the drift may have a negative sign.
% In other words, Lemma~2 from~\cite{LenglerSudholtWittAlgo21} applies when the transition probabilities to $p_{i, t} + 1/K$ and $p_{i, t} - 1/K$ are both changed to $\left(\frac{1}{2} \pm \alpha\right)2p_{i, t}(1-p_{i, t})$ where $\alpha \in [-O(1/\sqrt{V_t}), +O(1/\sqrt{V_t})]$ and the statement on the drift is changed to
% \[
%     \E(p_{i, t+1}-p_{i, t} \mid p_{i, t}) \in \left[-O(1) \cdot \frac{p_{i, t}(1-p_{i, t})}{K\sqrt{V_t}}, +O(1) \cdot \frac{p_{i, t}(1-p_{i, t})}{K\sqrt{V_t}}\right].
% \]

% The analysis in~\citet{LenglerSudholtWittAlgo21} bounds the variance of the \cga on \onemax from above and below by arguing that the frequencies give rise to a number of (pseudo-)independent random walks and then bounding the contribution of these random walks 
% %with a non-negative drift that is bounded by some maximum value that depends on $V_t$. The contribution of the 
% that have the characteristics described above (cf.\ Lemma~2 in~\cite{LenglerSudholtWittAlgo21}). 
A central argument from~\cite{LenglerSudholtWittAlgo21} is that the variance can be accurately described by studying the so-called \emph{lifetime contribution} of one frequency, which is the total contribution that the frequency makes to the variance while the frequency does not reach any border. The lifetime contribution is then bounded from above and below by 
%proving general bounds for a class of random walks 
using a worst-case perspective for the drift: in each iteration, the drift may be chosen arbitrarily from a range between $0$ and a maximum value that depends on $V_t$ (cf.\ Lemma~11 and~12 in~\cite{LenglerSudholtWittAlgo21}). 
This worst-case view was necessary to deal with dependencies between frequencies and the intricate feedback loops.
For these bounds it is crucial that the drift is always non-negative. To translate this approach to \cliff, one would have to allow negative drift values in addition to positive ones. This means that sudden changes between positive and negative drift values are possible and, ultimately, the worst-case bounds on the lifetime contribution become too weak to prove that the variance stabilizes to super-constant values.
%as used in the proof of Theorem~\ref{the:lower-bound-on-variance-onemax}. 
%and thus these arguments do not translate to \cliff. If the drift was negative for a long period of time, the analysis in~\cite{LenglerSudholtWittAlgo21} could be applied when swapping the meaning of zeros and ones appropriately. However, we cannot easily exclude that the drift can switch between positive and negative values. A worst-case perspective as used in Lemmas~11 and~12 of~\cite{LenglerSudholtWittAlgo21}, that is, allowing positive or negative drift values  yields rather trivial results that are not strong enough to prove Theorem~\ref{the:lower-bound-on-variance-onemax} for the \cga on \cliff. 
We conjecture that the worst-case view is too pessimistic here
as the real dynamics are unlikely to rapidly switch between regimes where the drift is noticeably positive and noticeably negative. Providing rigorous arguments remains a challenge for future work.
% , however, we do not see an easy solution to complete the proof.
% Hence translating Theorem~\ref{the:lower-bound-on-variance-onemax} from \onemax to \cliff is left as an open problem for future work.
%as the real dynamics are unlikely to rapidly switch between regimes where the drift is noticeably positive and noticeably negative. Proving this remains a challenge for future work.

% The most obvious difference in the search dynamics on \onemax and \cliff is that for \onemax the frequencies at the lower border will gradually leave towards their upper border.
% On \cliff, such a situation will only emerge if the \cga has managed to overcome the cliff. While the \cga has not yet overcome the cliff, we will typically have $\Theta(n)$ frequencies at the lower border and $\Theta(n)$ frequencies at the upper border. If an equilibrium distribution emerges similar to the one on \onemax and statements similar to the ones in Theorem~\ref{the:lower-bound-on-variance-onemax} could be proven for \cliff, the second statement could be applied over a much longer period of time. In this case we would want to choose $K = n^{\Omega(1)}$ such that the failure probabilities become exponentially small and the variance is bounded from below for an exponential time.

\subsection{Empirical Evidence}

\begin{figure*}[hbt]
\centering
\begin{subfigure}[b]{0.49\textwidth}
\includegraphics[width=\textwidth,height=5cm]{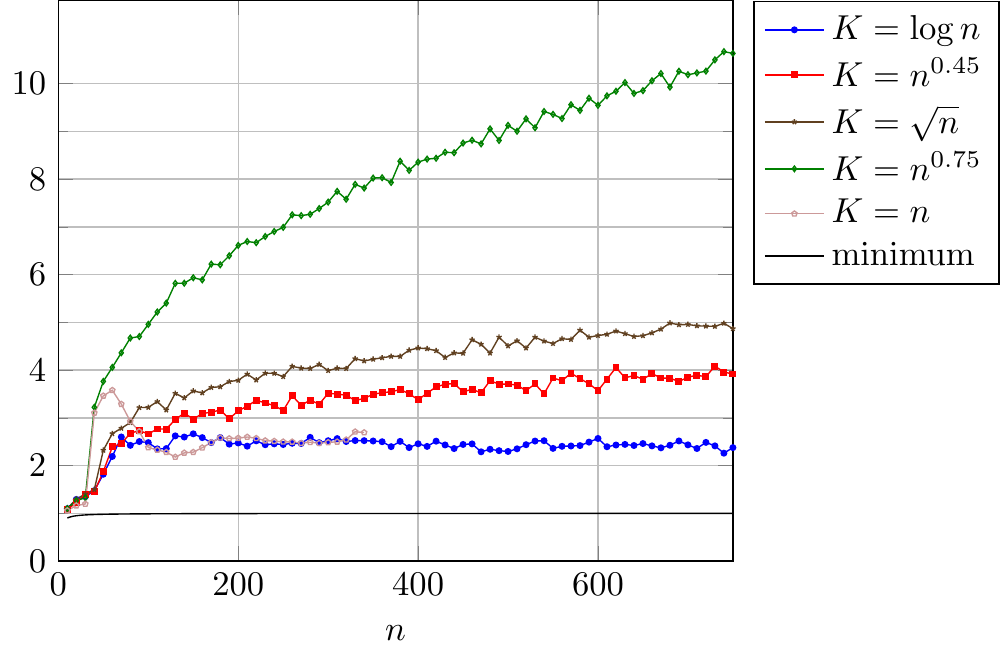}
\end{subfigure}
\hfill
\begin{subfigure}[b]{0.49\textwidth}
\includegraphics[width=\textwidth,height=5cm]{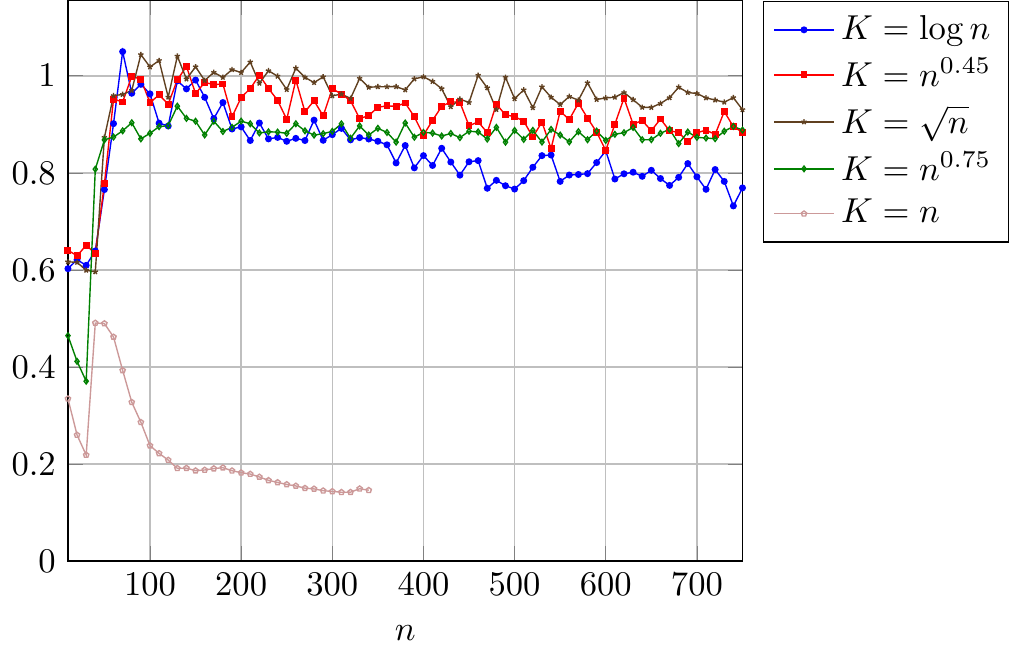}
\end{subfigure}
\vspace*{-0.9\baselineskip}
\caption{Variance after $100\sqrt{n}K+100K^2$ iterations, averaged over 100 runs, for increasing~$n$ and different values of $K$. The black line indicates the minimum variance of $1-1/n$. The plot on the right-hand side shows the same curves divided by $K^{1/2}$.}
\label{fig:variance-scaling-with-n}
\end{figure*}

We provide empirical evidence to support our belief that the variance is super-constant for interesting ranges of~$K$, including the ones from Theorem~\ref{the:lower-bound-on-variance-onemax}. We recorded the variance after $100\sqrt{n}K + 100K^2$ iterations and report averages taken over 100 runs, for increasing~$n$ and $K$ chosen as different functions of~$n$. The time bound is motivated by the upper bound of $O(\sqrt{n}K)$ for the expected optimization time of the \cga on \onemax for $K = \Omega(\sqrt{n} \log n)$~\cite{SudholtWittAlgo19} and the upper bound of $O(K^2)$ for the expected time for random walks reaching a border
(see~\cite{SudholtWittAlgo19} and Lemma~4 in~\cite{LenglerSudholtWittAlgo21}). Both results do not specify constants, hence we put a generous constant of 100 to enable the \cga to reach a state where the variance has stabilized.

% , for an unspecified leading constant. The expected time to reach a probabilistic model that samples around the \cliff area is certainly bounded by the expected time to sample the optimum of \onemax. The constant $100$ is chosen generously to give the \cga enough time to sample around the \cliff. An additional time of $100K^2$ is given. This is motivated by the fact that a random walk reaches a border in expected time~$O(K^2)$ (see~\cite{SudholtWittAlgo19} and Lemma~4 in~\cite{LenglerSudholtWittAlgo21}). Again, the implicit constant is not specified, so we choose a generous value of~100. We hope that the time limit is sufficient to enable the algorithm to reach a state where the variance has stabilized.

%The values of $K$ include: $K = \log n$, which is smaller than the range from Theorem~\ref{the:lower-bound-on-variance-onemax}; $K = n^{0.45}$, which is captured by Theorem~\ref{the:lower-bound-on-variance-onemax} as $K = n^{0.45} = O(\sqrt{n}/(\log(n) \log \log n))$; $K = \sqrt{n}$, which is just outside the range of Theorem~\ref{the:lower-bound-on-variance-onemax} and marks the upper boundary of the ``medium parameter'' regime~\cite{LenglerSudholtWittAlgo21}; $K = n^{0.75}$ and $K=n$ to capture larger polynomials outside the medium parameter regime as well.

The values of $K$ are chosen from $\{\log n, n^{0.45}, \sqrt{n}, n^{0.75}, n\}$. 
The value $K = n^{0.45}$ is captured by Theorem~\ref{the:lower-bound-on-variance-onemax} whereas $K=\log n$ and $K=\sqrt{n}$ are just outside the medium parameter regime. 
%We consider $K = n^{0.75}$ and $K=n$ to capture values outside the medium parameter regime as well.

On the left-hand side of Figure~\ref{fig:variance-scaling-with-n} we can see the variance scaling with~$n$. For $K = \log n$, the variance seems to remain constant. Runs for $K=n$ were only performed up to $n=340$ and the variance does not seem to increase, apart from a spike for small values of~$n$. (This spike persisted when increasing the time limit to $100K^2\log^2 n$.) All other values of $K$ yield curves that have a clear upward trend. The right-hand side shows the same data, normalized by dividing by $K^{1/2}$ and for all $K$ except $K=n$, the normalized values are strikingly close to~1. All curves but $K = \log n$ appear to be stable, suggesting that a variance lower bound of $\Omega(K^{1/2})$ might apply for medium~$K$.

% \dirk{Plan for experiments:
% \begin{itemize}
%     \item Consider $K = 0.3 \log n$ from the upper bound of the small $K$ regime, $\log^3n$ from Lengler et al., different small polynomials (e.g. $K^{0.25}$), $K^{0.45}$ from the regime of our semi-rigorous theoretical lower bound, $\sqrt{n}$, $n$ and possibly larger $K$ until the variance shoots up.
%     \item Record the final variance after $100(\sqrt{n}+K)K$ iterations (reduce the leading constant if this is infeasible). This is $\Omega(\sqrt{n}K)$ the time on \onemax and $\Omega(K^2)$ the time for completing a random walk on a neutral bit.
%     \item Plot the final variances for increasing~$n$. Increase $n$ in steps of 50. Take averages over 30 runs if possible.
%     \item Plot all average final variance curves. Plot the same data, normalized by $\log n$.
% \end{itemize}
% }

\section{A Semi-Rigorous Exponential Lower Bound}
\label{sec:semirigorousexpoentialbound}

As explained in Section~\ref{sec:superconstantvariance},
we believe that the sampling variance of the \cga
with potential~$P_t$ around the cliff 
is $\omega(1)$ and tends to lie stable
%around this value 
for
a long time. Under these
assumptions, we can formally prove that the
potential cannot efficiently cross the interval of negative
drift around $2n/3$, which was established above in
Theorem~\ref{the:negative-drift-cliff}.

The following lemma assumes a variance in
an interval $[\vl,\vu]$ while the potential is in the
drift interval $[2n/3-(\vl)^{1-\epsilon},2n/3]$. Under  
conditions on~$K$, $\vl$ and 
$\vu$, discussed after the proof, the 
%If $K\ge n^{2\epsilon'} \frac{\vu}{\vl}$ for some constant $\varepsilon' > 0$, 
time to pass the drift interval is exponential, with high probability.  %We will discuss these conditions 
% after the proof.

\begin{lemma}
\label{lem:app-neg-drift}
Assume that there are a constant $\epsilon'>0$, functions $\vl=\vl(n)=\omega(1)$, $\vu=\vu(n)\le n$
and a constant $c>1$
such that the property
$V_t\in [\vl,\vu]$ holds for all points in
time~$t$ where $P_t\in [a',b'] \coloneqq [2n/3-\vl^{1/2-\epsilon'},2n/3]$.
Assume $K\ge n^{\epsilon'}\vu/\vl^{1-\epsilon'}$ 
and $K=O(\sqrt{n})$ 
and define the hitting time
$T\coloneqq \min\{t\ge 0\mid P_t\ge b'\}$.
Then there
is a constant~$c'>0$ such that given $P_0\le a'$,
it holds that $\prob{T\le e^{c'K \vl^{1-\epsilon'}/\vu}}=2^{-\Omega(K 
\vl^{1-\epsilon'}/\vu)}$.
\end{lemma}

\begin{proof}
We verify the three conditions of the negative
drift theorem with scaling (Theorem~\ref{theo:simplified-drift-scaling}). 
Its parameters are chosen as $X_t=-KP_t$, $a=-Kb'$ and $b=-Ka'$.

For the first item, which deals with
a lower bound on the drift, we use
Theorem~\ref{the:negative-drift-cliff}  to  obtain (for $X_t=-KP_t$) the drift bound 
$\epsilon = \epsilon(n) = c_1 \sqrt{\vl}$ for 
a constant~$c_1>0$ within the
interval $[a,b]$ of length $\ell=b-a=K\vl^{1/2-\epsilon'}$.

To verify the second condition, we use
Corollary~\ref{cor:jump-pt-onestep} dealing
with the concentration of the one-step
change $\card{P_{t+1}-P_t}$ depending
on  the variance. Since 
we apply it to the scaled process $X_t=-KP_t$, the parameter~$\lambda$ is implicitly multiplied by~$K$. Hence,
there is an $r = c_2 \sqrt{\vu}$
for some sufficiently large constant~$c_2>0$ such that
%$r=\Theta(\sqrt{V_t}/K)$
    \[
    \prob{\lvert X_{t+1}-X_t\rvert\ge jr \mid \filt  \,;\, a< X_t}  \le  2e^{-\min\{ j^2 c_2^2, j c_2\sqrt{\vu} \}/3} \le e^{-j}
    \]
    for $j\in \N_0$, using that $\sqrt{\vu}\ge \sqrt{\vl}\ge 1$. Note that the condition
    $a<X_t$ is equivalent to
    $P_t<b'$ and $\lvert X_{t+1}-X_t\rvert = \lvert K(P_{t+1}-P_t)\rvert$.

We now analyze the exponent in the final bound on~$T$.  We have $r^2 = c_2^2 \vu$  and therefore
\[
\frac{\epsilon \ell}{132r^2} \ge
\frac{c_1  \sqrt{\vl} K\vl^{1/2-\epsilon'}}{132c_2^2 \vu  }
= \frac{c_1 K \vl^{1-\epsilon'}}{132c_2^2 \vu},
\]
which immediately leads to the exponent claimed
in the statement of this lemma
by setting $c'=c_1/(132c_2^2)$. Note that
the exponent becomes $\Omega(n^{\epsilon'})$
if $K\ge n^{\epsilon'}\vu/\vl^{1-\epsilon'}$.
However, we still have to verify the
third condition of the drift theorem.

First of all, we have $r=c_2\sqrt{\vu}\ge 1$ by 
choosing $c_2$ large enough and thus satisfy 
the lower bound on~$r^2$ in the third condition. Next, we note that $r/\epsilon\le \frac{c_2 \sqrt{\vu}}{c_1\sqrt{\vl}} = O(\sqrt{n})$ since $\vu\le n$  and therefore $\log(r/\epsilon)=O(\log n)$.
Hence, we can use the bound $\frac{\epsilon \ell}{132r^2} = \Omega(n^{\epsilon'})$, assuming  
 $K\ge n^{\epsilon'}\vu/\vl^{1-\epsilon'}$ and $K=O(\sqrt{n})$,  
from the previous paragraph to show that
$\frac{\epsilon \ell}{r^2 132 \log(r/\epsilon)}\ge 1$, 
which is equivalent to the upper bound on~$r^2$ 
in the third condition of the drift theorem. 
%The lower bound on $K$ is an explicit condition of
%this lemma, and the upper bound is implicit in the
%bound on the failure probability.
\end{proof}

To apply Lemma~\ref{lem:app-neg-drift}, we need %upper and
%lower 
bounds on the variance while $P_t$ is in the drift interval.
From the discussions above, we conjecture that the variance
stabilizes around a value~$v^*=\Omega(K^{1/2})$, which would allow
us to have $\vu/\vl=\Theta(1)$, hence $K\vl^{1-\epsilon'}/\vu =K(\vl/\vu)^{1-\epsilon'} /\vu^{\epsilon'} = 
\Omega(Kn^{-\epsilon'})$ (since $\vu\le n$). In this case, we
obtain an exponential bound already
for $K=\Omega(n^{2\epsilon'})$. If the variance is
allowed to oscillate
between $\vl$ and $\vu$ that are not of the
same asymptotic order, we can still apply the lemma
under reasonable assumptions.  If the variance is allowed 
to oscillate between a lower bound 
$\vl$ and an upper bound 
$\vu$ such that $\vu/\vl\le K^{1-\delta}$
for a constant $\delta>0$ (like, 
\eg, with the bounds $\Omega(K^{2/3})$ 
and $O(K^{4/3})$ appearing in 
Theorem~\ref{the:lower-bound-on-variance-onemax}), 
then 
$
{K \vl^{1-\epsilon'}}/{\vu} \ge
({K }/{K^{1-\delta}
)\vu^{-\epsilon'}} = K^{\delta}/\vu^{\epsilon'} 
\ge K^{\delta}/n^{\epsilon'}$,
%
%\frac{c K(K)^{(2/3)(1-\epsilon')}} {K^{4/3}}
which is still polynomially growing in~$K$ 
if, \eg, $K\ge n^{2\epsilon'/\delta}$. Here $\epsilon'>0$ can 
be chosen arbitrarily small.

\begin{figure*}[t]
\centering
\begin{subfigure}[b]{0.49\textwidth}
\includegraphics[width=\textwidth,height=5cm]{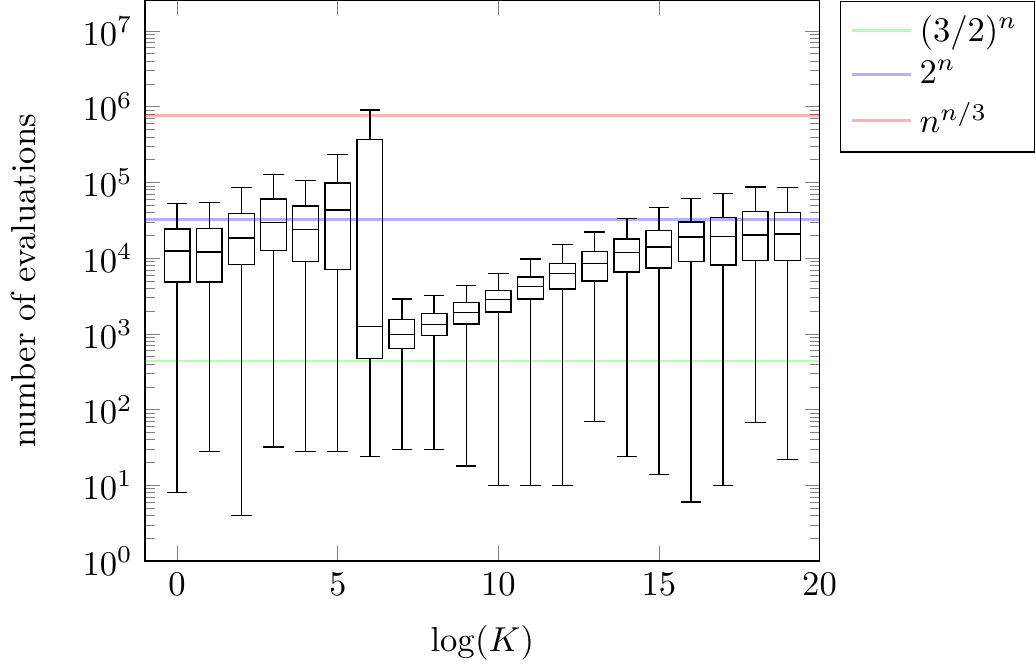}
\end{subfigure}
\hfill
\begin{subfigure}[b]{0.49\textwidth}
\includegraphics[width=\textwidth,height=5cm]{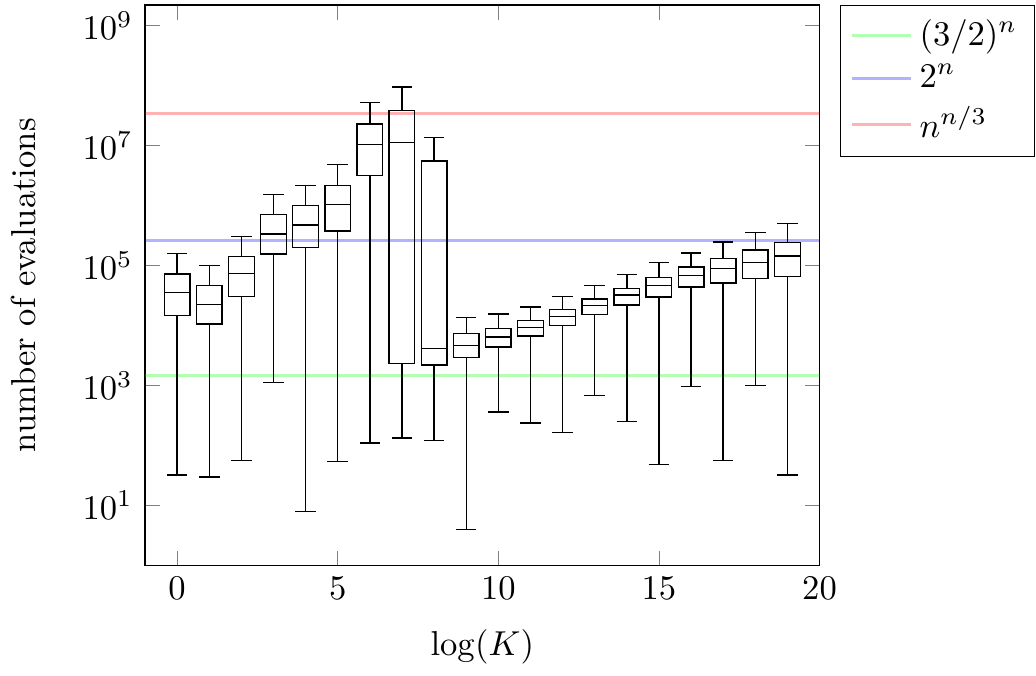}
\end{subfigure}
\vspace*{-0.9\baselineskip}
\caption{Box plots for the number of evaluations in 1000 runs on \cliff with $n=15$ (left) and $n=18$ (right) and exponentially increasing values of~$K$. The plots also show values of $(3/2)^n$, $2^n$ and $n^{n/3}$ for comparison.}
\label{fig:runtime-boxplots}
\end{figure*}

\paragraph{The potential does not jump over the negative drift interval.} To rule out that the \cGA optimizes \cliff efficiently, we also have to prove it unlikely
that the potential changes drastically
in one step and ``jumps over the drift interval'' $[2n/3-V_t^{1/2-\epsilon},2n/3]$.
However, it is not difficult to prove that the following event
is unlikely:
there is a point of time~$t$ where
$P_t<2n/3-(V_t)^{1/2-\epsilon}$
but $P_{t+1}>2n/3$. Since
$V_{t+1}\ge (1-1/K)V_t \ge V_t/2$, the length of the drift interval stemming from Theorem~\ref{the:negative-drift-cliff} is
at most halved in the transition from time~$t$ to~$t+1$. Hence,
if $P_{t+1}\le P_t +
(V_t/2)^{1/2-\epsilon}/2$, then $P_{t+1}\le 2n/3-(V_{t+1})^{1/2-\epsilon}/2$, \ie,
the process has not jumped over the interval. For simplicity,
we work
with the lower bound $(V_t)^{1/2-\epsilon}/4$ on the potential
difference in the following. We can then easily apply Corollary~\ref{cor:jump-pt-onestep} to
show that the variance does not change by
at least $(V_t)^{1/2-\epsilon}/4$ in a step
with overwhelming probability. We choose $\lambda=
K (V_t)^{1/2-\epsilon}/4$ and obtain
a failure probability of $e^{-\Omega(\min\{\lambda^2/V_t,\lambda\})} = e^{-\Omega(\min\{K^2/V_t^{2\epsilon},K V_t^{1/2-\epsilon}\})}$.
If we have  $K\ge n^{2\epsilon}$ then the
probability of increasing the potential
by $V_t^{1/2-\epsilon}/4 \le (V_t/2)^{1/2-\epsilon}/2$ in one step has probability $e^{-\Omega(K)}$, using $V_t\le n$.

Altogether, the analyses presented in this section show
that the potential, under reasonable assumptions on the
sampling variance, takes exponential time to exceed the
value~$2n/3$. If $P_t\le 2n/3$, sampling the optimum of
cliff, \ie, the all-ones string,
has probability at most $(3/4)^{n/9} =  2^{-\Omega(n)}$
since at least $n/9$
frequencies have to be below~$3/4$. Hence, the optimum
is not sampled in this situation with overwhelming probability.
%
%Taking all these arguments together, we have proven 
%the following theorem.
We have proven 
the following theorem.

\begin{theorem}
Assume the setting of Lemma~\ref{lem:app-neg-drift}. 
Then the optimization time of the \cGA on \cliff is 
$2^{\Omega(n^{\epsilon'})}$ with probability 
$1-2^{-\Omega(n^{\epsilon'})}$.
\end{theorem}

% Our exponential lower bounds (under assumptions) are of the kind
% $2^{\Omega(n^\epsilon)}$. Although such bounds are usually
% denoted as (weakly) exponential, our experimental supplements
% in the following section suggest that the runtime is in
% fact at least $2^{\Omega(n)}$ and for some parameter ranges
% even $n^{\Omega(n)}$.

\section{Empirical Runtimes \& Open Problems}
\label{sec:experiments}

Figure~\ref{fig:runtime-boxplots} shows boxplots highlighting the median runtimes and their distributions for $n=15$ and $n=18$ and $K$ being set to a power of~2 from $1$ to $2^{19}$. Note that the $y$ axis uses a logarithmic scale.

As can be seen, the parameter $K$ has a significant impact on the runtime and its parameter landscape seems complex.
For both problem sizes, the median runtime for small $K$ is close to~$2^n$.
We suspect that, for extreme updates, the \cga shows a chaotic behavior resembling random search (cf.\ Theorem~17 in~\cite{LenglerSudholtWittAlgo21} for \onemax). This is caused by extreme genetic drift~\cite{SudholtWittAlgo19,DoerrZhengIEEETEC20}.

As $K$ grows, the median runtime increases considerably, with some runs exceeding $n^{n/3}$ evaluations. For medium values of~$K$, the frequencies tend to reach their borders quickly due to genetic drift. If most of the frequencies remain at their borders and the potential is close to $2n/3$, roughly $n/3$ frequencies must be at the lower border and the probability of sampling the optimum from there is at most $n^{-n/3}$. 

This regime is followed by a sudden and steep drop at $K=2^6 = 64$ for $n=15$ and $K=2^8 = 256$ for $n=18$, respectively. When doubling $K$ one more time, the distribution is highly concentrated around the median, and runtimes are close to $(3/2)^n$. 
We suspect that in this parameter range, where $K$ is exponential in~$n$, the frequencies increase slowly and evenly. When the potential reaches $2n/3$ and all frequencies are similar to $2/3$, the \cga would have a probability of roughly $(2/3)^n$ to sample the optimal solution. 
This would explain the phase transition where we suspect the expected optimization time to drop from $n^{\Omega(n)}$ to $2^{O(n)}$ and possibly even to at most $c^n \cdot \mathrm{poly}(n)$ for a constant $3/2 \le c < 2$. 

When increasing $K$ even further, the frequencies remain so close to their initial values of $1/2$ that the \cga behaves like random search again. Indeed, the median runtime seems to approach $2^n$ for the largest values of~$K$ examined.

So, the best choice of $K$ seems to be in the sweet spot where the frequencies are able to rise equally towards a potential of $2n/3$ and stay there for long enough to sample the optimum. 

We finish with some open problems related to these observations.
\begin{openproblem}
Sharpen the variance bounds from~\cite{LenglerSudholtWittAlgo21} on \onemax,  Theorem~\ref{the:lower-bound-on-variance-onemax}, and add variance bounds for times in $[0, t_1]$.
\end{openproblem}

\begin{openproblem}
Prove rigorously that for the \cga on \cliff, with high probability, the variance is super-constant throughout an exponential period of time, for appropriate update strengths ${K = n^{\Omega(1)}}$.
\end{openproblem}

\begin{openproblem}
Prove a lower bound of $n^{\Omega(n)}$ for the \cga on \cliff for appropriate values of~$K$ below the observed phase transition.
\end{openproblem}

\begin{openproblem}
\label{problem:c-to-n-bound}
Prove an upper bound of $c^n \cdot \mathrm{poly}(n)$ for a constant $c < 2$ for exponential~$K$ beyond the observed phase transition.
\end{openproblem}

In order to address the open problem~\ref{problem:c-to-n-bound}, it might be necessary to prove that the frequencies tend to increase evenly. 
\begin{openproblem}
\label{problem:even-frequencies}
Prove that, for the \cga on \onemax or \cliff with large values of~$K$, the frequencies tend to increase evenly from their initial value of~$1/2$, and that they remain concentrated around the expectation for a period of time.
\end{openproblem}
Solving open problem~\ref{problem:even-frequencies} may not be as easy as it looks. Many standard concentration bounds do not apply since the frequencies are not independent and each step may have a large knock-on effect on future frequency dynamics, as discussed for the powerful method of bounded martingale differences in~\cite[Section~10.3]{DoerrProbabilisticTools}.

\section*{Acknowledgments}

This work has been supported by the Australian Research Council (ARC) through grant FT200100536 
and by the Independent Research Fund Denmark 
through grant DFF-FNU  8021-00260B.

\balance
\bibliographystyle{abbrvnat}
\bibliography{eda,references}

\clearpage

\end{document}